\documentclass[11pt]{article}
\usepackage[margin=1in]{geometry}

\usepackage{amsmath}
\usepackage{amsfonts}
\usepackage{amsthm}
\usepackage{mathtools}
\usepackage{caption}
\usepackage{subcaption}
\usepackage{float}

\usepackage{natbib}
\setcitestyle{authoryear,open={[},close={]}} %

\usepackage[pagebackref=true]{hyperref} 
\renewcommand*\backref[1]{\ifx#1\relax \else (Cited on #1) \fi}

\DeclareMathOperator*{\argmax}{arg\,max}
\DeclareMathOperator*{\argmin}{arg\,min}

\DeclareMathOperator*{\Support}{Support}

\newcommand{\E}{\mathop{\mathbb{E}}}
\newcommand{\ones}{\mathbf{1}}

\newcommand{\X}{\mathcal{X}}
\newcommand{\Y}{\mathcal{Y}}

\newcommand{\I}{\mathcal{I}}
\newcommand{\D}{\mathcal{D}}

\newcommand{\winner}{\textrm{winner}}
\newcommand{\simplemax}{\textrm{SimpleMax}}

\newcommand{\mcube}{\I^m}

\newcommand{\vect}[1]{\boldsymbol{#1}}
\newcommand{\smalldots}{\hbox to 1em{.\hss.\hss.}}
\newcommand{\ltwo}{\ell^2}
\newcommand{\linf}{\ell^\infty}

\newtheorem{theorem}{Theorem}
\newtheorem{condition}{Condition}
\newtheorem{corollary}{Corollary}
\newtheorem{assumption}{Assumption}
\newtheorem{lemma}{Lemma}
\newtheorem{definition}{Definition}

\usepackage{xcolor}
\definecolor{red}{RGB}{220,20,60}
\definecolor{blue}{RGB}{70,130,180}
\definecolor{orange}{RGB}{255,165,0}

\title{Hedging and Approximate Truthfulness in Traditional Forecasting Competitions}
\author {
    Mary Monroe, Anish Thilagar, Melody Hsu, Rafael Frongillo
    \\
    University of Colorado Boulder
}
\date{}

\begin{document}

\maketitle

\begin{abstract}
In forecasting competitions, the traditional mechanism scores the predictions of each contestant against the outcome of each event, and the contestant with the highest total score wins.
While it is well-known that this traditional mechanism can suffer from incentive issues, it is folklore that contestants will still be roughly truthful as the number of events grows.
Yet thus far the literature lacks a formal analysis of this traditional mechanism.
This paper gives the first such analysis.
We first demonstrate that the ``long-run truthfulness'' folklore is false: even for arbitrary numbers of events, the best forecaster can have an incentive to hedge, reporting more moderate beliefs to increase their win probability.
On the positive side, however, we show that two contestants will be approximately truthful when they have sufficient uncertainty over the relative quality of their opponent and the outcomes of the events, a case which may arise in practice.
\end{abstract}

\section{Introduction}

In a forecasting competition, multiple forecasters submit predictions on multiple events or held-out data points.
Prominent examples include geopolitical forecasting tournaments like the Good Judgement Project and the Hybrid Forecasting Competition, and data science competitions such as those hosted on Kaggle.
The most popular way platforms select a winner is through the traditional mechanism we call Simple Max: after scoring each forecaster's predictions against the realized event outcomes, the forecaster with the highest total score wins.

While intuitive, it is well-established that Simple Max is not a truthful mechanism \citep{lichtendahl2007probability, witkowski2023incentive}.
To maximize their chance of winning the competition, forecasters can have incentives to submit predictions which differ substantially from their true beliefs.
Despite its incentive issues, Simple Max continues to enjoy widespread popularity.
The mechanism's pervasiveness is in part due to folklore that the advantages of untruthfulness are attenuated by the large number of events relative to forecasters.
For example, \citet{aldous2019prediction} states that ``in the long run it is best to be `honest.'" 
However, the community lacks any theoretical analysis backing up this claim.
Indeed, we lack \emph{any} strategic analysis of the Simple Max mechanism beyond a single event, a major omission given the ubiquity of this traditional mechanism.

In this paper, we give the first strategic analysis of Simple Max for multiple events, with a goal of understanding when this long-run-truthfulness folklore does and does not hold.
We first identify a natural example where folklore fails: when the best forecaster is much better than the rest.
Here forecasters will not be truthful in equilibrium, even for a large number of events.
In particular, it is strictly better for the best forecaster to \emph{hedge} her report towards the others' reports in order to reduce the variance of her scores.

We next study the two-forecaster case and identify a regime where contestants \emph{will} be approximately truthful: when they both believe they have a shot at winning but have sufficient uncertainty about the report of their opponent.
These conditions can be difficult to achieve if the number of events is large, as the scores of forecasters with fixed skill will separate and the worse forecaster will no longer have a shot at winning.
Therefore we expect our result to apply in settings where the number of events is not too large, such as geopolitical forecasting competitions, or in data science competitions with sufficient noise in the leaderboard.

Taken together, our results show that behavior under the Simple Max mechanism is not as simple as folklore claims. 
If platforms care about eliciting true beliefs from contestants, our counterexample shows that there are settings they need to avoid.
Meanwhile, our approximate truthfulness result shows that, at least in some regimes, forecasters will behave in accordance with their beliefs.
We conclude with several other implications and future work.

\subsection{Related work}
Non-truthful behavior in forecasting competitions under Simple Max is well-documented in previous work. \citet{lichtendahl2007probability} first showed that forecasters extremize their reports when they want to maximize the chance they win under Simple Max.
\citet{witkowski2018incentive, witkowski2023incentive} note that no matter the total number of events, strategic forecasters may still extremize on some subset of them.
These works demonstrate the existence of non-truthfulness, but fall short of directly challenging folklore, as the results make no claims about the significance of these deviations in the long run.
For example, even if forecasters extremize on a small subset of events, their average deviation from truthful may still be small.
By contrast, we extend previous results to demonstrate consistent non-truthful behavior in larger competitions, and also show regimes where approximate truthfulness still holds.

\section{Model}
In a forecasting competition, there are $m$ independent binary events $Y_1, \dots, Y_m \in \{0, 1\}$, each with bias $\theta_t = \E[Y_t]$.
Let $\I = [0,1]$ be the unit interval, and $\mcube$ be the $m$-dimensional unit cube. 
There are $n$ forecasters; each forecaster $i$ submits report $r_{it} \in \I$ for event $t$.
Let $\vect{r}_i \in \mcube$ be the vector of forecaster $i$'s reports across events. 
We define $\vect{r}_{i,-t}$ as the vector of forecaster $i$'s reports for all events except $t$.
Each event $Y_t$ is sampled to obtain the outcome vector $\vect{y} \in \{0, 1\}^m = \Y$.
Finally, the competition uses some mechanism to declare a singular winner using the set of all reports $(\vect{r}_1, \ldots, \vect{r}_n)$ and outcomes $\vect{y}$.
Forecasters may play pure strategies, always choosing the same report $\vect r_i$, or mixed strategies, where they sample $\vect r_i$ from some distribution $\sigma_i$.
\begin{definition}[Strategy]
    A strategy for player $i$ is a probability distribution $\sigma_i \in \Delta(\I^m)$ over the set of possible report vectors from which $\vect r_i$ is sampled.
    A strategy profile $\vect{\sigma} = (\sigma_1, \dots, \sigma_n)$ is the vector of strategies for all players.
\end{definition}

\subsection{The Simple Max mechanism}
\label{sec:simple-max}
The traditional forecasting competition mechanism, which we call \emph{Simple Max}, scores each forecaster on every event and selects the forecaster with the highest aggregate score, breaking ties uniformly.

\begin{definition}[Simple Max]
  Given reports $(\vect{r}_1,\ldots,\vect{r}_n) \in (\mcube)^n$ and an outcome vector $\vect{y}\in\Y$, define the \emph{winning set} to be the set
  \begin{equation}
    \label{eq:winning-set}
    \winner(\vect{r_1},\smalldots,\vect{r}_n,\vect{y}) = \argmax_{i\in\{1,\smalldots,n\}} \sum_{t=1}^m S(r_{it},y_t)~.
  \end{equation}
  The \emph{Simple Max} mechanism is the randomized mechanism $\simplemax(\vect r_1,\ldots,\vect r_n,\vect y) \in \Delta_n$ that chooses a forecaster $i$ as the winner with probability
  \begin{equation}
    \label{eq:simple-max}
    \simplemax(\vect{r}_1,\smalldots,\vect{r}_n,\vect{y})_i = \frac{\ones\{i \in \winner(\vect{r}_1,\smalldots,\vect{r}_n,\vect{y})\}}{|\winner(\vect{r}_1,\smalldots,\vect{r}_n,\vect{y})|}.
  \end{equation}
\end{definition}
When considering strategic players, we will define their utility to be exactly this win probability $U_i (\vect{r}_i; \vect{r}_{-i}, \vect{y}) = \simplemax (\vect{r}, \vect{y})_i$.
We will exclusively focus on the case of the quadratic score, $S(r,y) = 1-(r-y)^2$, as is standard in the literature.
Because of this choice of $S$, the winning set is equivalent to the set of reports closest in Euclidean distance to $\vect{y}$, i.e., $\winner(\vect{r}_1,\ldots,\vect{r}_n,\vect{y}) = \argmin_i \|\vect{r}_i - \vect{y}\|_2$.
We will use this geometric perspective often.

\subsection{Approximate truthfulness}
To measure the degree of truthfulness of the mechanism, we quantify how far the forecasters' reports deviate from their beliefs.
In particular, we define two notions of the \emph{approximate truthfulness} of agents' reports, based on $\ltwo$ and $\linf$ distances.%
\footnote{These definitions of approximate truthfulness, as in \citet{frongillo2021efficient}, are stronger than the typical one which requires that \emph{utility} is approximately optimized by truthful reporting.}
\begin{definition}[Approximately truthful]
    A report vector $\vect r_i$ for forecaster $i$ is \emph{$\gamma$-$\ltwo$ approximately truthful} if $\tfrac{1}{\sqrt m} \|\vect{r}_i - \vect{p}_i\|_2 \leq \gamma$.
    A strategy $\sigma_i$ is a $\gamma$-$\ltwo$ truthful strategy for forecaster $i$ when it is  $\gamma$-$\ltwo$ truthful with probability 1, i.e.,
    each of its elements is a $\gamma$-$\ltwo$ approximately truthful strategy.
    A report $\vect r_i$ is \emph{$\gamma$-$\linf$ approximately truthful} if $\| \vect{r}_i - \vect{p}_i \|_{\infty} \leq \gamma$.
\end{definition}
Note that $\linf$ is a much stronger notion of approximate truthfulness than $\ltwo$, as it tightly bounds the maximum deviation of a report on any single event, whereas the first notion only requires that the total deviation across all reports is within some $\ltwo$ ball.

\subsection{Bayesian forecasters}
We consider a Bayesian belief model where forecasters receive some information about each event via a signal.
Formally, each forecaster $i$ receives a signal $X_{it} \in \X_{it}$ for event $t$.
We denote all of forecaster $i$'s signals by $\vect{X}_i = (X_{i1}, \dots, X_{im})$, and the set of signals for event $t$ by $\vect{X}_t = (X_{1t}, \dots, X_{nt})$.
The $X_{it}$ are independent across all $i$ and $t$.
Furthermore, all forecasters share a common prior $\E[Y_t]$ of the outcome $Y_t$ and the set of all $X_{it}$. 
When they receive their signal $x_{it} \sim X_{it}$, each forecaster Bayesian updates to their posterior $p_{it}(X_{it}) = \E[Y_t | X_{it} = x_{it}]$.
We define the ground truth probability as $\theta_t(\vect{x}_t) = \mathbb{E}[Y_t | \vect{X}_t = \vect{x}_t]$, the posterior of a Bayesian who has seen all signals.
We suppress the arguments of $p_{it}$ and $\theta_t$ when they are clear.

\section{Hedging: a Counterexample to Folklore} \label{sec:counterexample}
The community has long understood the lack of incentive compatibility under Simple Max for small numbers of events. 
Moreover, winning forecasters in real competitions have repeatedly reported more extreme probabilities to increase their chance of winning \citep{kaggle2017march}.
Fundamentally, reporting a more extreme report increases the variance of one's score, which can increase the likelihood of winning despite a decrease in the average score (Figure~\ref{fig:extremizing-hedging-pdfs}).
Folklore, such as \citet{aldous2019prediction}, posit that such behavior is an artifact of only having a small number of events, and forecasters should converge to truthful reporting as the number of events increases. 
Many real world forecasting competitions use Simple Max under the assumption that with enough events the best forecasters will still be truthful.

We show that this is not generally the case.
In particular, while bad forecasters may extremize, even good forecasters who perfectly know $\vec \theta$ may intentionally report less extreme probabilities to increase their chance of winning.
This tradeoff is similar to that of the extremizing forecasters, except now they are \emph{decreasing} the variance of their final score at the cost of lowering their expected score.
When a forecaster's expected score is already far ahead of the competition, this increases their win probability. 
To the best of our knowledge, we are the first to observe that such deviations, which we call \emph{hedging}, can be beneficial to knowledgeable players.
In particular, we will show that when they are far ahead, Bayesian forecasters will report something bounded away from their true belief in equilibrium. 
Contrary to folklore, this bound grows as the number of events increases. 

\begin{figure}
    \centering
    \includegraphics[width=0.35\textwidth]{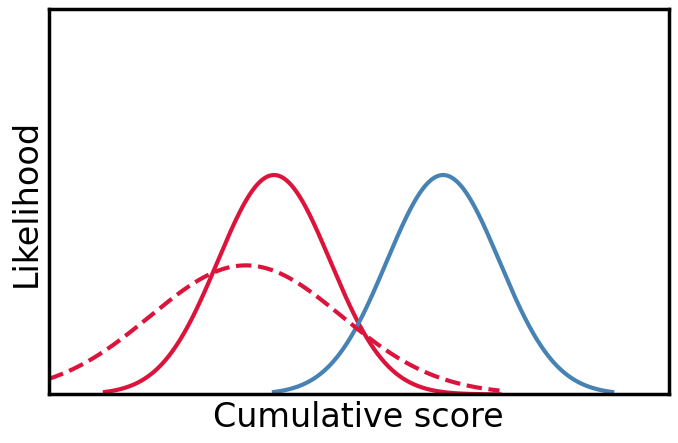}
    \includegraphics[width=0.35\textwidth]{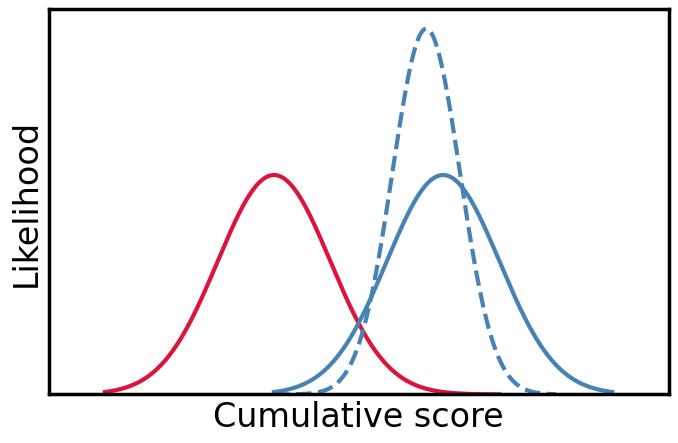}
    
    \caption{
        The score distributions of two (relatively) \textcolor{blue}{good} and \textcolor{red}{bad} forecasters.
        Their chance of winning is proportional to the region where their distributions overlap.
        On the left, the \textcolor{red}{bad} forecaster can extremize to the \textcolor{red}{dashed}: despite lowering her mean, this \emph{increases} her variance, and thus her win share.
        On the right, the \textcolor{blue}{good} forecaster can similarly increase her winshare by hedging to the \textcolor{blue}{dashed} distribution, which despite lowering her mean, also \emph{lowers} her variance.
        Intuitively, the reason the \textcolor{blue}{good} forecaster benefits hedging is that it decreases the variance of their score, ``locking in" their lead, even while decreasing their expected score.
        To draw a familiar analogy from sports: a team which is behind will start making long-shot attempts to score, increasing their variance, while the team which is ahead will try slow down the game, decreasing their variance as well as their chance of scoring more.
    }
    \label{fig:extremizing-hedging-pdfs}
\end{figure}

\subsection{Setting}
We use a specific instance of the Bayesian model with $p$-biased coins for some $p \in (0, 1/2)$.
Note that we can always observe the same behavior without the Bayesian assumption, but it is less clear that all players will play equilibrium strategies. %
Each event $Y_t$ is a random coin flip whose bias is $\theta_t \in \{p, 1-p\}$ chosen independently and uniformly at random.
The prior satisfies $\E [ Y_t ] = 1/2$ for all $t$.
We fix some $i$ to be the \emph{informed} forecaster, while all other forecasters $j \neq i$ remain uninformed.
Each such forecaster $j$ always receives \emph{no} signal for all events, so their belief remains $\vect p_{j} = \vect c := (1/2, \dots, 1/2)$, the center of the $m$-dimensional hypercube.
However, forecaster $i$ receives a signal that is exactly the true probabilities $\X_i = \vect \theta \in \{p, 1-p\}^m$.
This gap in information, and knowledge of it, is key to the hedging behavior we observe.

For mathematical convenience and without loss of generality, we invert the outcome of all events $t$ satisfying $\theta_t = 1-p$ and let $\vect \theta = \vect p = (p, \dots, p)$.
Any specific set of reports $r_{-i}$ for $j\neq i$ is then equivalent to the mixed strategy that randomly applies a reflection of the $m$-hypercube%
\footnote{Formally, any involution in the hyperoctahedral group $B_m$.}
to $r_{-i}$, since all such reflections are equally likely. 
We consider mixed strategy equilibria as pure strategy equilibria do not generally exist, even for simple cases like $m = n = 2$. 
We are particularly interested in when players play strategies that are close to their beliefs, which we measure using the notion of approximate truthfulness in Definition 2.

\subsection{Exact equilibria}
For sufficiently small choices of $m$ and $n$, we can exactly characterize the equilibrium.
We first start with an extremely simple case: that of a single event ($m=1)$.

We first consider just $n = 2$ forecasters.
Here, the equilibrium is for $i$ to report $r_i = 0$, and $j$ to maximally extremize and choose $r_j \in \{0, 1\}$ uniformly at random. 
This gives the forecasters win probabilities of $U_i = 3/4 - p/2$ and $U_j = 1/4 + p/2$ respectively.
If $i$ deviates to any report in $(0, 1]$ they can only win with probability $\leq 1/2 < 3/4 - p/2$, and if $j$ deviates to any report in $(0, 1)$ they can only win with probability $\leq p < 1/4 - p/2$.

As we increase the number of forecasters, the incentive to extremize only increases. 
$i$ will continue to play $r_i = 0$ and win with probability
\begin{align*}
    \tfrac{1}{2^{n-1}}\left[\sum_{k=0}^{n-1} \frac{ \binom{n-1}{k}}{k+1} + \frac{p}{n+1}\right]
    &= \frac{2}{n} - \tfrac{1}{2^{n-1}}\left[\frac{1}{n} - \frac{p}{n+1}\right] ~,
\end{align*}
while forecasters $j\neq i$ report a vertex at random and evenly split the remaining winshare.

When $m=2$, the equilibrium gets much more complex.
We describe it for $n = 2$ forecasters when $p \in (1/3, 1/2)$, as shown in Figure~\ref{fig:2d-equilibrium}.
In this case, $i$ will play $\vect r_i = \vect c = (1/2, 1/2)$, the center of the unit square, with probability $\frac{1-p}{2-p}$, and choose some $\vect r_i \in \{(0, 1/2), (1/2, 0)\}$ with probability $\frac{1}{2(2-p)}$ each.
Forecaster $j$ will play $\vect r_j = \vect c$ with probability $3 \cdot \frac{p}{2-p}$, and each of $\vect r_j \in \{1/2 \pm 1/4\}^2$ with probability $\frac{1/2-p}{2-p}$.
\begin{figure}[t]
    \centering
    \includegraphics[scale=1]{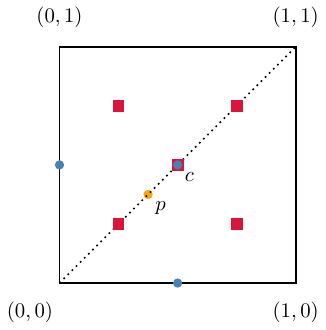}
    \caption{
        The equilibrium strategies of $i$ and $j$ for $m=n=2$ and $p \in (1/3, 1/2)$. 
        The {\textcolor{blue}{blue circles}} denote points $i$ plays and the {\textcolor{red}{red squares}} denote $j$'s.
        Notably both players play in the center $\vect c$ some of the time but $i$ always plays bounded away from their belief $\vect p$ ({\textcolor{orange}{orange point}}).
    }
    \label{fig:2d-equilibrium}
\end{figure}
While forecaster $j$ is extremizing with some non-zero probability, their average report is still truthful.
On the other hand, forecaster $i$ alternates between hedging to the center and extremizing to the midpoints.
Their average report is $\left(\frac{3-2p}{4(2-p)},\frac{3-2p}{4(2-p)}\right)$.
So when $p < \frac{5-\sqrt{13}}{4}$, $i$'s average report is hedged towards the center. 
Otherwise, their average report is slightly extremized.

\subsection{Hedging} \label{subsec:hedging}
As the number of events and players grow larger, it remains an open question to find an equilibrium even in this specific setting.
Our focus instead will be on determining whether an approximately truthful equilibrium can exist.
In Corollary~\ref{cor:no-truthful-equilibrium} we answer this question negatively: an equilibrium cannot be $\ltwo$-approximately truthful even in the limit of large $m$.
In particular, Theorem~\ref{thm:hedging-strict-dominance} shows that by hedging and reporting $\vect r^* = (p^*, \dots, p^*)$, where $p^* = \frac{1/2 + p}{2}$, forecaster $i$ can only increase their win probability. 
The Theorem holds even for relatively large values of $\epsilon$.
For example, when $p = 0.9$ and $m \geq 32$, deviating by $0.2$ in every coordinate is strictly better than \emph{any} $0.04$ $\ltwo$-approximately truthful strategy.

Our proof is geometric, leveraging the fact that a forecaster wins when their report is the closest to the outcome in Euclidean distance, as discussed in \S~\ref{sec:simple-max}.
Given any $\epsilon$-$\ltwo$ approximately truthful reports $\vect r$ for all forecasters, we show that hedging to $\vect r^*$ is a strictly dominant strategy for $i$.
Specifically, given any vertex $\vect y$, we show that either $\vect r^*$ is closer to $\vect y$ than any $\vect r_j$ for $j\neq i$, or $\vect r^*$ is closer to $\vect y$ than $\vect r_i$ is to $\vect y$.
See Figure~\ref{fig:hedging}.
The implication is that $\vect r^*$ wins on a strict superset of outcomes that $\vect r_i$ wins.

As shorthand, let us define the following distances, all from a report vector to some vertex $\vect y$:
$d_p(\vect y) = \|\vect p - \vect y\|_2$, 
$d^*(\vect y) = \|\vect r^* - \vect y\|_2$, 
$d_i(\vect y) = \|\vect r_i - \vect y\|_2$,
$d_{j}(\vect y) = \|\vect r_j - \vect y\|_2$.

\begin{figure}[t]
    \centering
    \includegraphics[width=0.6\linewidth, trim={0 0 4cm 0}, clip]{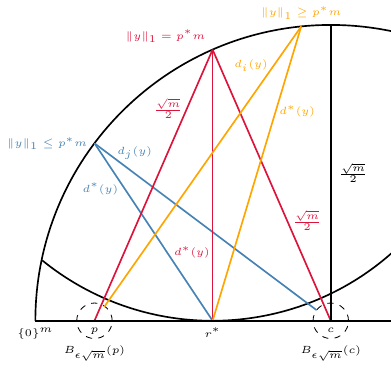}
    \caption{
        The geometry of the plane containing $\vect p$, $\vect c$, $\vect r^*$ and some $\vect y$. 
        The bottom line is the main diagonal of the $m$-dimensional hypercube that goes from $\{0\}^m$ to $\{1\}^m$.
        Note that $\vect r^*$ is the midpoint of $i$'s belief $\vect p$ and the belief of the other forecasters $\vect c$, so $i$ is moving to the middle of the ``information gap.''
        The radius of the balls around $\vect p$ and $\vect c$ are chosen so that they are bounded away from the ball of radius $d^*(\vect y)$ around a $\vect y$ that lies perpendicular to the diagonal at $\vect r^*$ ({\textcolor{red}{red}}).
        Lemma~\ref{lem:counterexample-lower-ineq} considers $\|\vect y\|_1 \leq p^* m$ ({\textcolor{blue}{blue}}), while Lemma~\ref{lem:counterexample-upper-ineq} considers $\|\vect y\|_1 \geq p^* m$ ({\textcolor{orange}{orange}}). 
        The high level idea is that forecaster $i$ can shift from winning roughly just the $\vect y$ to the left of $\vect r^*$ when reporting $\vect r_i$, to additionally winning some fraction of the $\vect y$ between $\vect r^*$ and $\vect c$ by hedging to $\vect r^*$.
    }
    \label{fig:hedging}
\end{figure}

To bound these distances away from each other, we choose $\epsilon$ such that the balls of radius $\epsilon \sqrt{m}$ are bounded away from the arc of radius $d^*(\vect y)$ about $\|\vect y\|_1 = p^*m$ (the arc about the {\textcolor{red}{red}} point in Figure~\ref{fig:hedging}).  
Therefore, we choose $m$ large enough such that there is enough space in the circle of radius $\sqrt{m}/2$ for the arc to lie above $\vect c$ and $\vect p$. 
Furthermore, to show strict dominance the arc must be at least distance 2 from the balls.
(Note that all the distances scale as $\sqrt{m}$, so the gap between the arc and the balls gets proportionally smaller for large $m$).
Then, we must choose $p$ such that there is enough space between the arc and the main diagonal for a gap of distance 2 to exist.
The following condition suffices for all of these relationships to hold.

\begin{condition} 
    \label{cond:hedging-bounds} ~

    \begin{itemize}
        \item $m \geq 21$ 
        \item $0 < p < \frac{1}{2} - 2\sqrt{\tfrac{2}{\sqrt{m}}(1 - \tfrac{2}{\sqrt{m}})}$
        \item $0 < \epsilon < \frac{1}{2} - \sqrt{p^* ( 1- p^*)} - \frac{2}{\sqrt{m}}$ ~.
    \end{itemize}
\end{condition}
Note that such an $\epsilon$ will always exist when the first two inequalities hold.
To show that $\vect r^*$ dominates $\vect r_i$ we consider two cases.
When $\vect y$ lies to the left of $\vect r^*$, $\vect r^*$ wins against each $\vect r_j$, as shown by the {\textcolor{blue}{blue}} lines in Figure~\ref{fig:hedging}, so $\vect r_i$ cannot strictly dominate $\vect r^*$.
\begin{lemma} \label{lem:counterexample-lower-ineq}
    Suppose $m$, $p$ and $\epsilon$ satisfy Condition~\ref{cond:hedging-bounds}.
    Then for any $\vect y \in \Y$ with $\|\vect y\|_1 \leq p^* m$, if $\|\vect r_j - \vect c\|_2 < \epsilon \sqrt{m}$ then $d^*(\vect y) < d_j(\vect y) - 2$.
\end{lemma}
When $\vect y$ lies to the right of $\vect r^*$, $\vect r^*$ dominates $\vect r_i$, as shown by the {\textcolor{orange}{orange}} lines in Figure~\ref{fig:hedging}.
\begin{lemma} \label{lem:counterexample-upper-ineq}
    Suppose $m$, $p$ and $\epsilon$ satisfy Condition~\ref{cond:hedging-bounds}.
    Then for any $\vect y \in \Y$ with $\|\vect y\|_1 \geq p^* m$, if $\|\vect r_i - \vect p \|_2 < \epsilon \sqrt{m}$ then $d^*(\vect y) < d_i(\vect y) - 2$.
\end{lemma}
Both proofs are deferred to \S~\ref{appendix:counterexample}.
Jointly, they show that when all players report $\epsilon$-$\ltwo$ truthfully, $\vect r^*$ dominates $\vect r_i$. 
\begin{lemma} \label{lem:counterexample-dominance}
    Suppose $m$, $p$ and $\epsilon$ satisfy Condition~\ref{cond:hedging-bounds} and each forecaster chooses an $\epsilon$-$\ltwo$ approximately truthful report.
    Then for every $\vect y \in \Y$, $U_i(\vect r_i, \vect r_{-i}, \vect y) \leq U_i(\vect r^*, \vect r_{-i}, \vect y)$.
\end{lemma}

\begin{proof}
    If $\|\vect y\|_1 \leq p^* m$, then by Lemma~\ref{lem:counterexample-lower-ineq} for every $j \neq i$ we obtain $d^*(\vect y) < d_j(\vect y)$.
    Therefore, $\vect y$ is closer in Euclidean distance to $\vect r^*$ than $\vect r_j$, so $i$ is the winner with probability 1 when playing $\vect r^*$; thus $U_i(\vect r^*, \vect r_{-i}, \vect y) = 1 \geq U_i(\vect r_i, \vect r_{-i}, \vect y)$.

    Otherwise, if $\|\vect y\|_1 \geq p^*m$, then by Lemma~\ref{lem:counterexample-upper-ineq} $d^*(\vect y) < d_i(\vect y)$.
    Therefore, $\vect y$ is closer in Euclidean distance to $\vect r^*$ than $\vect r_i$, so $U_i(\vect r^*, \vect r_{-i}, \vect y) \geq U_i(\vect r_i, \vect r_{-i}, \vect y)$.
\end{proof}

Furthermore, because of the strict gap of distance 2 between the balls, there must be some $\vect y$ where $\vect r^*$ strictly dominates $\vect r_i$.
\begin{lemma}\label{lem:counterexample-strict-dominance}
    Suppose $m$, $p$ and $\epsilon$ satisfy Condition~\ref{cond:hedging-bounds} and each forecaster chooses an $\epsilon$-$\ltwo$ truthful report.
    Then there exists some $\vect y \in \Y$ such that 
    $U_i(\vect r_i, \vect r_{-i}, \vect y) < U_i(\vect r^*, \vect r_{-i}, \vect y)$.
\end{lemma}
\begin{proof}
    Let $\Y_i = \{\vect y \mid U_i(\vect r_i; \vect r_{-i}, \vect y) = 1\}$ and $\Y^* = \{\vect y \mid U_i(\vect r^*; \vect r_{-i}, \vect y) = 1\}$.
    It suffices to show that $\Y^* \setminus Y_i \neq \emptyset$.
    
    Suppose there is some $\vect y \not \in \Y_i$ such that $\|\vect y\|_1 < \lceil p^* m \rceil$.
    Then, by Lemma~\ref{lem:counterexample-lower-ineq} $d^*(\vect y) < d_j(\vect y)$ for all $j \neq i$, so $\vect y \in \Y^*$ and $\Y^* \setminus Y_i \neq \emptyset$ so we are done.
    
    Otherwise, if no such $\vect y$ exists, we pick $\vect y \in \argmax_{\vect y' \in \Y_i} \|\vect y'\|_1$.
    We must have $\|\vect y\|_1 \geq \lceil p^* m \rceil$.
    Now, choose any $\vect y^+ \not \in Y_i$ such that $\|\vect y^+\|_1 = \|\vect y\|_1 + 1$ and $\|\vect y^+ - \vect y\|_1 = 1$. 
    (i.e. choose $\vect y^+$ by taking $\vect y$ and flipping some 0 to a 1).
    Such a $\vect y^+$ must exist since $\vect y$ was the $\argmax$ and $\{1\}^m \not \in Y_i$.

    Since $\|\vect y^+ - \vect y\|_1 = 1$, we must have $d_i(\vect y^+) - 1 \leq d_i(\vect y)$ and similarly $d_j(\vect y) - 1\leq d_j(\vect y^+)$. 
    But since $\vect y \in \Y_i$, $d_i(\vect y) > d_j(\vect y)$.
    Combining, we obtain 
    \begin{align*}
        d_j(\vect y^+) 
        &\geq d_j(\vect y) - 1 > d_i(\vect y) - 1 \geq d_i(\vect y^+) - 2 > d^*(\vect y) ~,
    \end{align*}
    by Lemma~\ref{lem:counterexample-upper-ineq}.
    Therefore, $\vect y^+ \in \Y^*$, so $\Y^* \setminus Y_i \neq \emptyset$.
\end{proof}

Taken together, the previous results imply that hedging to $\vect r^*$ strictly dominates $\vect r_i$ when any $\epsilon$-$\ltwo$ approximately truthful strategy is played.
\begin{theorem}\label{thm:hedging-strict-dominance}
    Suppose $m$, $p$ and $\epsilon$ satisfy Condition~\ref{cond:hedging-bounds}.
    Let $\vect r_i$ be any $\epsilon$-$\ltwo$ approximately truthful report for forecaster $i$ and  $\sigma_{-i}$ be an $\epsilon$-$\ltwo$ approximately truthful strategy profile for all $j \neq i$.
    Then $\vect r_i$ is strictly dominated by hedging and reporting $\vect r^*$.
\end{theorem}

\begin{proof}
    By Lemma~\ref{lem:counterexample-dominance}, for any $\vect r_{-i} \in \Support(\vect \sigma_{-i})$, and every $\vect y\in \Y$, $U_i(\vect r_i; \vect r_{-i}, \vect y) \leq U_i(\vect r^*; \vect r_{-i}, \vect y)$.
    Additionally, by Lemma~\ref{lem:counterexample-strict-dominance}, for at least one $\vect y$ the inequality is strict.
    Therefore, 
    \begin{align*}
        \E_{\vect y \sim \vect \theta} \left[U_i(\vect r_i; \vect r_{-i}, \vect y) \right]
        &< \E_{\vect y \sim \vect \theta} \left[U_i(\vect r^*; \vect r_{-i}, \vect y) \right] ~.
    \end{align*}

    Taking the expectation of both sides over all $\vect r_i \sim \sigma_i$ shows that forecaster $i$'s win probability when reporting $\vect r_i$ is strictly less than when they hedge and report $\vect r^*$. 
    
\end{proof}

In particular, Condition~\ref{cond:hedging-bounds} implies $p^* > p + \epsilon$, so $i$ will always strictly benefit by hedging to $\vect r^*$ from any $\epsilon$-$\ltwo$ approximately truthful strategy. 
Therefore, no $\epsilon$-$\ltwo$ approximately truthful equilibrium exists.
\begin{corollary} \label{cor:no-truthful-equilibrium}
    Suppose $m$, $p$ and $\epsilon$ satisfy Condition~\ref{cond:hedging-bounds}.
    Then, every equilibrium strategy profile $\vect{\sigma}$ is not $\epsilon$-$\ltwo$ approximately truthful.
\end{corollary}

\section{Achieving Approximate Truthfulness} \label{sec:approx-truthfulness}
Though we have shown regimes where forecasters deviate from their beliefs, we now prove Simple Max is still approximately truthful in settings which may arise in practice. 
These results do not identify specific equilibria, and instead place conditions on the beliefs about the reports of their opponents.
Our results are practical in actual competitions, in the sense that forecasters do form beliefs about the reports of their opponents (either by reasoning about equilibrium behavior or looking at historical data~\cite{kaggle2017march}) and those beliefs could plausibly satisfy these conditions.

Our main result, Theorem~\ref{thm:2-forecaster-truthful}, considers two forecasters $i$ and $j$, and shows that both will report approximately truthfully when they have sufficient uncertainty over each others' reports.
Specifically, we require that (1) the expected score difference is smooth, (2) any report vector induces enough variance in the score, and (3) forecaster $i$ does not think her score will be much larger than $j$'s, or vice versa.
In real-world competitions, all 3 conditions are quite reasonable.
We use this uncertainty to show that each player's utilities are approximately affine in their score.
Therefore, to maximize their utility, they will roughly aim to maximize their score.
This leads to \emph{leave-one-out approximate truthfulness}: fixing event $t$, the optimal report satisfies $r_{it} \approx p_{it}$ for \emph{any} fixed report vector $\vect{r}_{i,-t}$ across events $t' \neq t$.
The strength of that guarantee leads to the strong $\linf$ version of approximate truthfulness.
However, significant groundwork is required to quantify ``enough'' uncertainty, and demonstrate how it propagates from beliefs to utilities to approximately truthful reports.
The remainder of this section outlines our strategy for proving Theorem~\ref{thm:2-forecaster-truthful}; see \S~\ref{appendix:approx-truthfulness} for the full proof.

\subsection{Notation}
Let the random variable $R_{jt} \in [0, 1]$ be $j$'s report for event $t$, and let $\vect{R}_j = (R_{j1}, R_{j2}, \smalldots, R_{jm})$ be $j$'s random report vector. 
For any forecaster $i$, let $\vect{R}_{-i,t} = (\vect{R}_1, \smalldots, \vect{R}_{i-1}, \vect{R}_{i+1}, \smalldots, \vect{R}_n)$ be the random vector of reports for all agents $j \neq i$.
In this section, we model forecaster $i$'s belief as a joint distribution $\D_i$ over event outcomes $\vect{Y}$ and others' reports $\vect{R}_{-i,t}$, as in \citet{witkowski2023incentive}.
We denote $\vect{p}_i = (p_{i1}, \ldots, p_{im})$ as the marginal distribution under $\D_i$ over outcomes $\vect{Y}$.
We use this belief model partly as a shorthand in notation, but it is also a more general version of the Bayesian setting introduced earlier and thus encapsulates a broader range of scenarios.

\subsection{Utilities}
To distinguish between different utility functions, we will use subscripts to denote the variable of interest and overlines to denote expectations.
Let $\overline{U}_i(\vect{r}_i) = \E_{\D_i} U_i (\vect{r}_i; \vect{R}_j, \vect{Y})$ be forecaster $i$'s expected utility under her belief distribution as a function of report vector $\vect{r}_i$.
We also define $U_{it}(r_{it}; \vect{r}_{i,-t}, \vect{R}_j, \vect{Y}) = U_i(\vect{r}_i; \vect{R}_j, \vect{Y})$ as a function of just $r_{it}$, fixing $\vect{r}_{i,-t}$.
Let 
\[\overline{U}_{it} (r_{it}; \vect{r}_{i,-t}, R_{j}, Y) =  \E_{(\vect{R}_{j}, \vect{Y}) \sim \D_i}  U_{it} (r_{it}; \vect{r}_{i,-t}, \vect{R}_j, \vect{Y})\]
be $i$'s expected utility as a function of $r_{it}$, over all outcomes and reports of forecaster $j$.
When the right hand terms are clear we will simply use $U_{it}(r_{it})$ and $\overline{U}_{it}(r_{it})$.

For any $t$, let $S(r_{it}, R_{jt}, Y_t)$ $= S(R_{jt}, Y_t) - S(r_{it}, Y_t)$ denote the difference between $j$ and $i$'s score on event $t$.
We impose a smoothness assumption on the sum of score differences $\sum_t S(r_{it}, R_{jt}, Y_t)$ to ensure that the probability of a tie is zero.

\begin{assumption} \label{assumption:smoothness}
    For every forecaster $i$ and any report $\vect{r}_i \in \I^m$, their expected score difference CDF
    $G_i(x; \vect{r}_i) = \Pr_{D_i}\left[\sum_t S(r_{it}, R_{jt}, Y_t) \leq x\right]$ is absolutely continuous.
\end{assumption}

As long as one event's score difference is absolutely continuous, the assumption is satisfied. 
Though it may seem unrealistic to assume a distribution over reports is continuous, note that our analysis would easily extend to settings without this assumption since ties occur with vanishingly small probability in practice; see \S~\ref{appendix:ties} for a full discussion.

Under Assumption~\ref{assumption:smoothness}, $\Pr_{\D_i} \left[|\winner(\vect{r}_i,\vect{R}_j,\vect{y})| > 1\right] = 0$.
It follows that each forecaster's expected utility is simply their probability of being in the winning set under Simple Max.
That is, the expected utility satisfies $\overline{U}_i(\vect{r}_i) = G_i(0; \vect{r}_i).$
Moreover, $\overline{U}_{it} (r_{it})$ corresponds to the expected probability that $-S(r_{it}, R_{jt}, Y_t)$ is greater than the cumulative score difference across events $t' \neq t$.
Specifically, if
\[G_{it}(x) = \Pr_{(\vect{Y}_{-t}, \vect{R}_{j,-t}) \sim \D_i}  \left[\sum_{t' \neq t} S(r_{it'}, R_{jt'}, Y_{t'}) \leq x\right]\]
is the CDF of the score difference distribution for $t' \neq t$,
\begin{equation} %
    \overline{U}_{it}(r_{it}) = \E_{(R_{jt}, Y_t) \sim \D_i} G_{it} \left( -S(r_{it}, R_{jt}, Y_t) \right) ~.
\end{equation}

\subsection{Approximate affineness}

We first derive sufficient conditions for forecaster $i$ to be $\gamma$-$\linf$ approximately truthful on event $t$, given a fixed report vector $\vect{r}_{i,-t}$.
In particular, we show that forecaster $i$'s utility function $G_{it}$ is \emph{approximately affine}.

\begin{definition}
    Function $F$ is \emph{$(\beta, \alpha, \epsilon$)-approximately affine} when
    $$\sup_{x \in [-1, 1]} \left| F(x) - (\beta x + \alpha) \right| \leq \epsilon.$$
\end{definition}

Note that we are evaluating $G_{it}$ at the score difference $S(r_{it}, Y_t) - S(R_{jt}, Y_t)$.
Since $G_{it}$ is approximately affine, forecaster $i$ is (roughly) maximizing her score.
By properties of proper scoring rules, she will report (roughly) truthfully. 
In the following lemma, we quantify the worst-case deviation of her report to her belief as a function of approximate affineness. 

\begin{lemma} \label{lem:approx-linear-implies-approx-truthful}
    If $G_{it}$ is $(\beta, \alpha, \epsilon$)-approximately affine for each event $t$, agent $i$ is $\gamma$-$\linf$ approximately truthful for $\gamma = \sqrt{2\epsilon/\beta}$.
\end{lemma}

\begin{proof}
    Let $\hat{\epsilon}(X) = G_{it}(X) - (\beta X + \alpha)$ be the random variable representing the error of approximating $G_{it}(X)$ with an affine function, where $X \in [0, 1]$ with probability one.
    Then,
    \begin{align*}
        \overline{U}_{it}(r_{it}) 
        &= \E_{\D_i} \left[ \alpha - \beta  S(r_{it}, R_{jt}, Y_t) + \hat{\epsilon} \left( -S(r_{it}, R_{jt}, Y_t) \right) \right] \\
        &= \alpha + \beta \E_{Y_t \sim p_i} \bigl[ \E_{R_{jt} | Y_t} - S(r_{it}, R_{jt}, Y_t)\bigr] + \E_{\D_i} \left[ \hat{\epsilon} \left( - S(r_{it}, R_{jt}, Y_t) \right) \right] ~.
    \end{align*}
    With probability 1, $|\hat{\epsilon}(X)|$ $\leq \epsilon$.
    Then, using $S(p_{it}, R_{jt}, Y_t) - S(r_{it}, R_{jt}, Y_t) = p_{it}^2 - r_{it}^2 + 2Y_t (r_{it} - p_{it})$, we can strictly bound the difference in utilities
    \begin{align*}
        \overline{U}_{it}(p_{it})  -  \overline{U}_{it}(r_{it})  &\geq \beta  \E_{Y_t \sim p_i} \kern-0.6em \left[ p_{it}^2 - r_{it}^2 + 2Y_t (r_{it} - p_{it}) \right]  -  2\epsilon \\ 
        &= \beta(p_{it} - r_{it})^2 - 2\epsilon \\
        &> \beta \left({2\epsilon/\beta}\right) - 2\epsilon \\
        &= 0.
    \end{align*}
    Thus, $r_{it}$ is strictly dominated by $p_{it}$ if $|r_{it} - p_{it}| > \gamma$, so agent $i$ will report $\gamma$-$\linf$ approximately truthfully.
\end{proof}

Figure \ref{fig:approx-truthful} gives intuition for when this result is useful: specifically, approximating $G_{it}$ by a Normal CDF, we achieve small $\gamma$ when the mean of the cumulative score difference has small magnitude relative to the variance. 

\begin{figure}[t]
    \centering
      \includegraphics[width=0.55\linewidth]{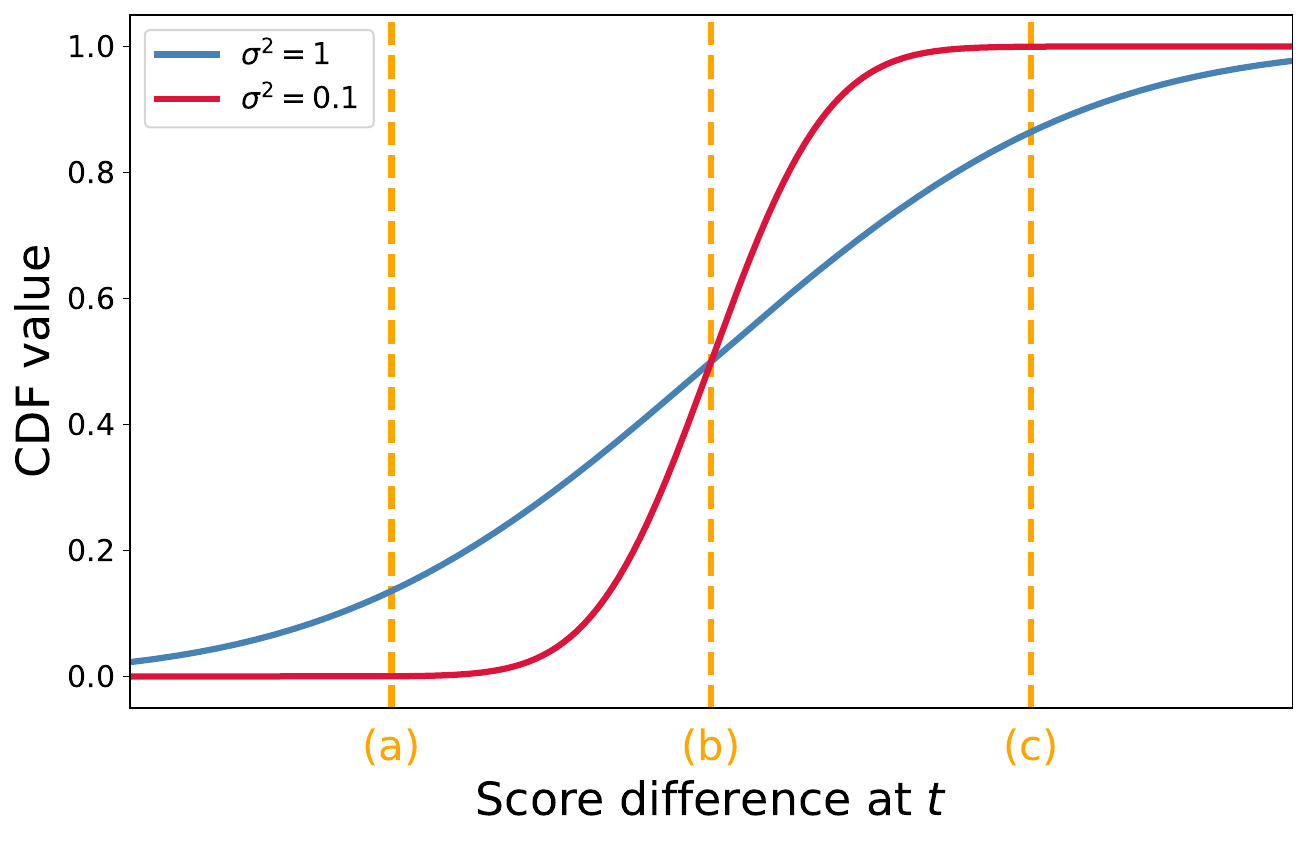}
    \caption{
    Two examples of Normal CDFs approximating the expected utility $G_{it}$.
    If the score difference at event $t$ (in $[-1, 1]$) occurs at (a) or (c), approximate affineness ceases to hold as the CDF flattens out.
    We observe small error of the affine approximation around (b), the mean; thus we want its magnitude to be small.
    Note also the area where the CDF resembles an affine function is much narrower for the distribution with lower variance (\textcolor{red}{red}).
    }
    \label{fig:approx-truthful}
\end{figure}

\paragraph{Edgeworth expansions.}
We next aim to characterize forecaster $i$'s utility function over her report $r_{it}$.
Since $G_{it}$ is the CDF of a sum of independent random variables, we can apply the Central Limit Theorem (CLT).
Unfortunately, the convergence rate of $G_{it}$ to a normal distribution is not tight enough to control the approximate affine error $\epsilon$ in Lemma \ref{lem:approx-linear-implies-approx-truthful} \citep{esseen1942}.
We thus turn to the theory of \emph{Edgeworth expansions}.
An Edgeworth expansion can be thought of as a tighter version of the CLT. 
While we know $G_{it}$ eventually approaches a normal distribution, we can achieve a better convergence rate by adding on more terms \citep{petrov1975independent}.
These terms depend on higher cumulants of the distribution; for example, the first-order Edgeworth expansion includes a function depending on the third moment. 
We can bound the $\linf$ distance between $G_{it}$ and its second-order Edgeworth expansion $E_{it}$ under some conditions on the variance and smoothness of score differences:

\begin{condition} \label{cond:edgeworth}
    For forecaster $i$'s strategic report $\vect{r}_i$ and all $t$,
    \begin{enumerate}
        \item (Uncertainty) the score differences $S(r_{it}, R_{jt}, Y_t)$ have variance uniformly bounded below by some constant. 
        \item (Smoothness) the probability density of each score difference $S(r_{it}, R_{jt}, Y_t)$ is a function of bounded variation.\footnote{Absolute continuity or continuous differentiability are both sufficient conditions for bounded variation.}
    \end{enumerate}
\end{condition}
We include a weaker condition which only requires smoothness of some score differences in \S~\ref{appendix:approx-truthfulness}.

\paragraph{Leave-one-out approximate truthfulness.}
We now have the tools to characterize forecaster $i$'s utility function $G_{it}$ as $(\beta, \alpha, \epsilon)$-approximately affine.
In particular, we derive the values of $\beta, \alpha$ and $\epsilon$ for the Edgeworth expansion $E_{it}$ of $G_{it}$. 
To do so, we take the first-order Taylor expansion around 0 to derive the slope $\beta$ and intercept $\alpha$.
We can then bound the error $\epsilon_1$ of the affine Taylor function by the Lagrange remainder.
We also have the convergence rate of $G_{it}$ to its Edgeworth expansion, i.e. $\|G_{it} - E_{it}\|_{\infty} \leq \epsilon_2$.
By the triangle inequality, then, $G_{it}$ is $(\beta, \alpha, \epsilon_1 + \epsilon_2)$-approximately affine.
It follows by Lemma \ref{lem:approx-linear-implies-approx-truthful} that forecaster $i$ will be $\gamma$-$\linf$ approximately truthful for $\gamma = \sqrt{\tfrac{2(\epsilon_1 + \epsilon_2)}{\beta}}$ on event $t$ (under Condition \ref{cond:edgeworth}). 
Note that $\gamma$ depends on the mean and variance of score differences $\{ S(r_{it'}, R_{jt'}, Y_{t'}) \}_{t' \neq t},$ which in turn depends on the report vector $\vect{r}_{-i,t}$.

\subsection{Approximate truthfulness}
Our leave-one-out result gives a sufficient condition on the value of $\gamma$ for approximate truthfulness on event $t$, under \emph{any} fixed report vector $\vect{r}_{i,-t}.$
In order to reach our main result, then, we impose conditions on the entire vector $\vect{r}_i$ such that the condition for each event $t$ is met.
We define $\sigma_i^2(\vect{r}_i)$ and $P_i(\vect{r}_i)$ as the variance and absolute third moment of the aggregate score difference $\sum_t (S(r_{it}, R_{jt}, Y_t))$ induced by $\vect{r}_i$. 
Let $P_i = \max_{\vect{r}_i} P_i(\vect{r}_i)$ and  $\sigma_i = \min_{\vect{r}_i} \sigma_i(\vect{r}_i)$.

At a high level, the first condition implies there is enough randomness (variance) in scores from forecaster $i$'s perspective.
Meanwhile, the latter two conditions imply that forecaster $i$ thinks she is \emph{good}, as the expected utility of her belief is bounded below, but not \emph{too} good, as the expected utility of any report vector is bounded above.

\begin{condition} \label{cond:final-thm-2-forecasters}
    ~
    
    \begin{enumerate}
        \item $\sigma_i \geq 4$.
        \item For some $\delta \in (0, 1)$, $\overline{U}(\vect{p}_i) \geq  \tfrac{1}{2} - \delta$; and for any report vector $\vect{r}_i$, $\overline{U}(\vect{r}_i) \leq \tfrac{1}{2} + \delta$.
        \item $\tfrac{P_i}{\sigma_i^3} + \delta \leq 0.33$.
    \end{enumerate}
\end{condition}

Note that $\tfrac{P_i}{\sigma_i^3} = O\left(m^{-1/2}\right)$, so we expect Condition~\ref{cond:final-thm-2-forecasters} is satisfied more easily as $m$ grows.
With that bound, and assuming $\tfrac{P_i}{\sigma_i^3}$ is small, the latter two items imply forecaster $i$'s expected utility is roughly bounded to the interval $[0.2, 0.8]$.

We can now state our approximate truthfulness result.
\begin{theorem} \label{thm:2-forecaster-truthful} 
When Conditions \ref{cond:edgeworth} and \ref{cond:final-thm-2-forecasters} hold, forecaster $i$ is $\gamma$-$\linf$ approximately truthful with $\gamma = O \left(\sigma_i^{-1/2}\right).$
\end{theorem}

Note that $\sigma_i^2 = O(m)$, so that we expect $\gamma$ to decrease like $O(m^{-1/4})$.
At first glance, then, Theorem \ref{thm:2-forecaster-truthful} seems to imply that the $\linf$ distance of forecasters' reports converge to their beliefs as the number of events grows larger.
This is true to an extent, but recall that Condition \ref{cond:final-thm-2-forecasters} requires that the forecasters believe they are competitive with each other. 
In a typical setting where both forecasters have fixed skills, e.g.\ with forecaster $i$ being slightly better at predicting the weather than forecaster $j$, for large $m$ forecaster $j$ can no longer be competitive: each forecaster's aggregate score would concentrate about its mean and the distribution of forecaster $i$'s score would far surpass $j$'s.
Thus, the conditions for approximate truthfulness in practice may require $m$ to be small enough for the best forecaster to have peers.

We expect that similar results would hold for $n > 2$ forecasters, though analysis becomes much more difficult since utilities now correspond to beating the \emph{maximum} cumulative score difference across forecasters.
We expect the resulting $\gamma$ to grow quickly in $n$, as even when all forecasters have the same skill, when $n$ is larger, they each expect to perform much worse than the maximum of the rest.
This extension would thus be consistent with the fact that good forecasters still extremize in practice; see \S~\ref{sec:discussion}.

\section{Discussion} 
\label{sec:discussion}

This paper presents the first strategic analysis of traditional forecasting competitions beyond small numbers of events.
We first refute long-standing claims about long-run truthful behavior in forecasting competitions, via a counterexample with no approximately truthful equilibrium even for an arbitrarily large number of events.
This example suggests that the best forecaster benefits by hedging, moving her predictions closer to the prior.
We then show that two forecasters will be approximately truthful when close together in skill and given sufficient uncertainty about the other's reports.

\paragraph{Implications for practice.}
Let us summarize our results with a practical conjecture of how strategic behavior will play out in real competitions.
It seems rare that a single forecaster will be far ahead of the rest; a more typical scenario might have a pack of good forecasters with a long tail of lower-skilled forecasters.
It might be tempting to assume that our results in \S~\ref{sec:approx-truthfulness} will dominate for the pack of good forecasters, yielding approximate truthfulness in that group, with increasing incentive to extremize for lower-skilled forecasters.
Yet as alluded to at the end of \S~\ref{sec:approx-truthfulness}, unless the pack of good forecasters is of size 2 only, each skilled forecaster will face the maximum score of the rest, again giving an incentive to extremize (see Figure~\ref{fig:extremizing-hedging-pdfs} where now the \textcolor{blue}{blue} distribution is the maximum score of the competition).
This rough conjecture matches behavior seen in practice, where skilled (and indeed, winning) forecasters extremize their true beliefs \citep{kaggle2017march, alexander2022who}.

That said, our results suggest ways to alleviate these incentives.
Even when the number of events is quite large, contestants may deviate significantly from being truthful when they know they are far ahead or far behind.
In data science and machine learning competitions, this phenomenon could mean limiting the amount of information revealed in the leaderboard---not just for the usual information-theoretic reasons of preventing contestants from learning too much about the test/evaluation split of the data set, but because of the strategic properties of Simple Max.
Finally, to eliminate the extremizing among the most skilled forecasters, it could be beneficial to run a truncated version of the Soft Max mechanism of~\citet{frongillo2021efficient} which drops low scoring forecasters, with the analysis of \S~\ref{sec:approx-truthfulness} allowing a higher parameter $\eta$ than suggested by that paper.

\paragraph{Robustness of forecaster selection.}
The designer of a forecasting competition may naturally have two goals: (1) collect accurate predictions, and (2) select (one of) the best forecasters as the winner.
Thus far we have mainly focused on (1), which may be the most important when the forecasts are aggregated or used directly to make timely predictions, as in geopolitical competitions.
Yet in some settings goal (2) may be the primary objective.
Interestingly, while we have seen that Simple Max can fail wildly in goal (1), even for arbitrarily large numbers of events, our results are consistent with it still succeeding in goal (2).%
\footnote{A more realistic model of effort may paint a worse picture even for (2), as time is spent strategizing instead of improving forecasts.}
That is, even as forecasters deviate substantially from their beliefs, it has remained true that a best forecaster wins with high probability.
Specifically, we conjecture that with probability $1-\delta$, when the number of events is $m=\Omega\left(\tfrac{\log(n/\delta)}{\epsilon^2}\right)$, Simple Max selects the best forecaster in some Bayes-Nash equilibrium.

\paragraph{Future work.}
Aside from addressing any of the specific conjectures posited above, perhaps the most pressing need is to understand the equilibrium of Simple Max for $m > 2$.
Such an analysis is missing even for the specific example in \S~\ref{sec:counterexample}; while our results suggest that the bad forecasters will extremize while the good forecaster hedges, it is unknown if both behaviors are stable in equilibrium.
Perhaps the first step would be to understand the equilibrium in a ``full information'' setting, where the true probabilities are common knowledge.
Extensions of this paper's results to forecasting competitions that employ different scoring rules, such as log loss, may also hold interesting insights.

\subsection*{Acknowledgements}
We are grateful to Jens Witkowski and Siddarth Srinivasan for their insights and feedback.
We would also like to thank 
Adam Bloniarz,
Gülce Kardes,
Ezra Karger,
Bo Waggoner, 
and
Brian Zaharatos
for helpful discussions and comments.
This material is based upon work supported by the National Science Foundation under Grant No. IIS-2045347.

\bibliography{references}

\appendix

\section{Proofs omitted from Section \ref{sec:counterexample}}
\label{appendix:counterexample}

Recall that
$d_p(\vect y) = \|\vect p - \vect y\|_2$, 
$d^*(\vect y) = \|\vect r^* - \vect y\|_2$, 
$d_i(\vect y) = \|\vect r_i - \vect y\|_2$,
and $d_{j}(\vect y) = \|\vect r_j - \vect y\|_2$ be the $\ltwo$ distances from $\vect y$ to 
$\vect p$, $\vect p^*$, $\vect r_i$ and $\vect r_j$ for any $j \neq i$ respectively.

Additionally, we will use the fact that
\begin{align*}
    (d^*(\vect y))^2 
    &= \|\vect y\|_1 (1 - p^*)^2 + (m-\|\vect y\|_1)(p^*)^2 \\
    &= m (p^*)^2 + \|\vect y\|_1(1 - 2p^*),
\end{align*}
and similarly, $(d_p(\vect y))^2 = mp^2 + \|\vect y\|_1(1 - 2p)$.

\subsection{Proof of Lemma~\ref{lem:counterexample-lower-ineq}}
\begin{proof} 
    By the triangle inequality,
    \begin{align*}
        d_j(\vect y) &\geq \| \vect c - \vect y \|_2 - \epsilon \sqrt{m} \\        
        &= \sqrt{m}\left(\frac{1}{2} - \epsilon\right) \\
        &> \sqrt{m}\sqrt{p^*(1 - p^*)} + 2\\
        &= \sqrt{m (p^*)^2 + m p^*(1 - 2p^*)} + 2\\
        &\geq \sqrt{m (p^*)^2 + \|\vect y\|_1(1 - 2p^*)} + 2\\
        &= d^*(\vect y) + 2~.
    \end{align*}
\end{proof}

\subsection{Proof of Lemma~\ref{lem:counterexample-upper-ineq}}
\begin{proof}
    When $\|\vect y\|_1 = p^* m$, 
    \begin{align*}
        d_p(\vect y) - \epsilon \sqrt{m} 
        &= \sqrt{m p^2 + m p^* (1 - 2p)} - \epsilon \sqrt{m} \\ 
        &= \sqrt{m} \left( \frac{1}{2} - \epsilon \right).
    \end{align*}
    Meanwhile, 
    \begin{align*}
        d^*(\vect y) &= \sqrt{mp^2 + p^*m (1 - 2p^*)} \\
        &= \sqrt{mp^*(1 - p^*)} \\
        &< \sqrt{m} \left( \frac{1}{2} - \epsilon \right) - 2.
    \end{align*}

    Note that $d_p(\vect y)$ and $d^*(\vect y)$ can also be written as a function of just $\|\vect y\|_1$.
    Therefore, we define 
    \[f(x) = \sqrt{mp^2 + x (1 - 2p)} - \sqrt{m (p^*)^2 + x(1 - 2p^*)} \]
    so that $f(\|\vect y\|_1) = d_p(\vect y) - d^*(\vect y)$.
    Now, consider the derivative of $f$ at any $x = \|\vect y\|_1$
    \[\frac{df}{dx} = \left(\frac{1}{2} - p\right) \left[\frac{1}{d_p(\vect y)} - \frac{1}{2d^*(\vect y)} \right]. \]
    Since $p < \frac{1}{2}$,
    \begin{align*}
        4(d^*(\vect y))^2 - (d_p(\vect y))^2 
        &=  m \left(\frac{1}{4} + p \right) + \|\vect y\|_1 \left( \frac{1}{2} - p \right) \\
        &> 0,
    \end{align*}
    so $\frac{df}{dx} > 0$.
    Therefore, $f$ is increasing in $x$ and $d_p(\vect y) - d^*(\vect y)$ increases as $\|\vect y\|_1$ increases.
    By the triangle inequality $d_i(\vect y) - d^*(\vect y) \geq (d_p(\vect y) - \epsilon \sqrt{m}) - d^*(\vect y) > 2$ for $\|\vect y\|_1 = p^*m$.
    It follows that $d_i(\vect y) - d^*(\vect y) > 2$ for $\|\vect y\|_1\geq p^*m$.
\end{proof}

\section{Approximate Truthfulness Derivations (Omitted from \S~\ref{sec:approx-truthfulness})} \label{appendix:approx-truthfulness}
\subsection{Notation}
We first introduce additional notation needed to analyze forecaster $i$'s utility function.
Consider the sum of independent random variables $\sum_{t=1}^m X_t$ with mean $\mu$ and variance $\sigma^2$.
Let $\phi(x)$ and $\Phi(x)$ be the PDF and CDF respectively of the normal distribution with mean 0 and variance 1.
Let $H_l(x)$ be the Hermite polynomial of degree $l$, and denote $\kappa_l$ as the $l$-th order cumulant of the sum of $X_1, X_2, \ldots, X_m$. 
Note that up to $l = 3$, the $l$-th order cumulant corresponds to the central moment (so $\kappa_3$ is the sum of third moments). 
We denote the ratio between the $l$-th order cumulant and the variance $\sigma^2$ as $C_l = \frac{\kappa_l}{\sigma^2}.$

Now, fix an event $t$.
Let $\mu_{it}$ and $\sigma_{it}^2$ be the mean and variance, respectively, of the sum of score differences $\sum_{t' \neq t} S(r_{it'}, R_{jt'}, Y_{t'})$ under $i$'s belief distribution $\D_i$ for a fixed set of reports $\vect{r}_{i,-t}$.
We suppress the fixed vector $\vect{r}_{i,-t}$ as a parameter for simplicity.
We denote $\kappa_{l, it}$ to be the $l$-th order cumulant of the sum of score differences $\sum_{t' \neq t} S(r_{it'}, R_{jt'}, Y_{t'})$; similarly, we have $C_{l,it} = \frac{\kappa_{l,it}}{\sigma_{it}^2}.$

We also define moments of the cumulative score difference across \emph{all} events as a function of a report vector $\vect{r}_i$:
\begin{align*}
    \mu_i(\vect{r}_i) &= \sum_t \E \left[ S(r_{it}, R_{jt}, Y_t)\right], \\
    \sigma_i^2(\vect{r}_i) &= \sum_t \E \left( S(r_{it}, R_{jt}, Y_t) - \E[S(r_{it}, R_{jt}, Y_t)] \right)^2, \\
    P_i(\vect{r}_i) &= \sum_t \E |S(r_{it}, R_{jt}, Y_t) - \E[S(r_{it}, R_{jt}, Y_t)]|^3;
\end{align*}
i.e. $\mu_i(\vect{r}_i)$ is the mean, $\sigma_i^2(\vect{r}_i)$ is the variance, and $P_i(\vect{r}_i)$ is the third absolute moment.
Let $P_i = \max_{\vect{r}_i} P_i(\vect{r}_i)$ and  $\sigma_i = \min_{\vect{r}_i} \sigma_i(\vect{r}_i)$.

Finally, we define $\kappa_{l, i}$ to be the \emph{maximum} $l$-th order cumulant over report vectors $\vect{r}_i$ of the sum of score differences $\sum_{t = 1}^m S(r_{it}, R_{jt}, Y_{t})$, and we denote $C_{l,i} = \frac{\kappa_{l, i}}{\sigma_i^2}.$

\subsection{Edgeworth Expansions}
We now introduce the theory of Edgeworth expansions, following \citet{petrov1975independent}, and use this to characterize forecaster $i$'s expected utility function $G_{it}$.
Let $E(x)$ be the Edgeworth expansion up to two terms of the sum of independent random variables $X_1, X_2, \ldots, X_m$, where $\mu$ is the mean of these variables and $\sigma^2$ is the variance. 
That is, 
\[E(x) = \Phi \left(\frac{x - \mu}{\sigma}\right) + \frac{Q_1\left(\frac{x - \mu}{\sigma}\right)}{\sqrt{m}} + \frac{Q_2\left(\frac{x - \mu}{\sigma}\right)}{m},\]
where 
\begin{align*}
    \frac{Q_1(x)}{\sqrt{m}} &= - \frac{1}{ \sqrt{2\pi}} e^{-x^2/2} \left( \frac{\kappa_3}{\sigma^3} \right) \frac{H_2(x)}{6}, \\
    \frac{Q_2(x)}{m} &= \frac{1}{\sqrt{2\pi}} e^{-x^2/2} \left( \left( \frac{\kappa_3}{\sigma^3}\right)^2 \frac{H_5(x)}{72} + \left( \frac{\kappa_4}{\sigma^4}\right) \frac{H_3(x)}{24} \right).
\end{align*}
Let $E_{it}(x)$ be the Edgeworth expansion for independent random variables $\{ S(r_{it'}, R_{jt'}, Y_t') \}_{t' \neq t}$.

We first state a weaker condition which is satisfied by Condition \ref{cond:edgeworth} for convergence of $G_{it}$ to its Edgeworth expansion.
Note that we only require smoothness of a subsequence of score differences.

\begin{condition}{\cite{petrov1975independent}} \label{cond:edgeworth-weak}
    For forecaster $i$'s strategic report $\vect{r}_i$ and all events $t$,
    \begin{enumerate}
        \item (Uncertainty) the score differences $S(r_{it}, R_{jt}, Y_t)$ have variance uniformly bounded below by some constant. 
        \item There is a subsequence of score differences $\{ S(r_{it_k}, R_{jt_k}, Y_{t_k}) \}_{k=1}^{m^*}$ with $m^* \leq m$ terms such that 
        \begin{enumerate}
            \item (Nontrivial size) $\lim \inf \frac{m^*}{m^{\lambda}} > 0$ for some $\lambda > 0$.
            \item (Smoothness) the density $f_k$ of each difference $S(r_{it_k}, R_{jt_k}, Y_{t_k})$ is a function of bounded variation.
        \end{enumerate}
    \end{enumerate}
\end{condition}

Now we cite the convergence rate of $G_{it}$ to its second-order Edgeworth approximation $E_{it}$ under Condition \ref{cond:edgeworth-weak}:
\begin{lemma}{\citep{petrov1975independent}} \label{lem:edgeworth} 
Given Condition \ref{cond:edgeworth-weak} holds, for sufficiently large $m$ and some constant $D_{it}$,
$\|G_{it} - E_{it} \|_{\infty} \leq D_{it}/m$.
\end{lemma}

We can then characterize the values $\beta, \alpha, $ and $\epsilon$ such that the Edgeworth expansion $E(x)$ is $(\beta, \alpha, \epsilon)$-approximately affine.

\begin{lemma} \label{lem:approx-linear}
$E$ is $(\beta, \alpha, \epsilon)$-approximately affine with 
\begin{align}
    \beta &= \frac{1}{\sigma} \phi\left(\frac{-\mu}{\sigma}\right) \biggl( 1 + \frac{C_3}{6 \sigma} H_3 \left(\frac{-\mu}{\sigma} \right) + \frac{C_3^2}{72 \sigma^2}  H_6 \left(\frac{-\mu}{\sigma} \right) + \frac{C_4}{24 \sigma^2} H_4 \left( \frac{-\mu}{\sigma} \right) \biggr), \label{eq:beta} \\
    \alpha &= E(0) \label{eq:alpha}, \\
    \epsilon &= \max_{z \in [-1, 1]} \frac{1}{\sigma^3} \phi \left(\frac{z - \mu}{\sigma}\right) \biggl| (z - \mu) - \frac{C_3}{6} H_4 \left(\frac{z - \mu}{\sigma}\right) - \frac{C_3^2}{72 \sigma} H_7 \left(\frac{z - \mu}{\sigma}\right) - \frac{C_4}{ 24 \sigma} H_5 \left(\frac{z - \mu}{\sigma} \right) \biggr| \label{eq:epsilon}. 
\end{align}
\end{lemma}

\begin{proof}
    First, note that by properties of Hermite polynomials, $\frac{d}{dx} e^{-x^2/2} H_m(x) = - e^{-x^2/2} H_{m+1}(x)$.
    Following \citet{petrov1975independent},
    \begin{align*}
        \frac{Q_1'(x)}{\sqrt{m}} &= \frac{1}{ \sqrt{2\pi}} e^{-x^2/2} \frac{C_3}{6 \sigma} H_3(x), \\
        \frac{Q_2'(x)}{m} &= \frac{1}{\sqrt{2\pi}} e^{-x^2/2} \left( \frac{C_3^2}{72 \sigma^2} H_6(x) + \frac{C_4}{24 \sigma^2} H_4(x) \right), \\
        \frac{Q_1''(x)}{\sqrt{m}} &= -\frac{1}{ \sqrt{2\pi}} e^{-x^2/2} \frac{C_3}{6 \sigma} H_4(x), \\
        \frac{Q_2''(x)}{m} &= -\frac{1}{\sqrt{2\pi}} e^{-x^2/2} \left( \frac{C_3^2}{72 \sigma^2} H_7(x) + \frac{C_4}{24 \sigma^2} H_5(x) \right).
    \end{align*}
    
    By taking the Lagrange error of the first order Taylor approximation of $E$ around 0, we have $\sup_{x \in [-1, 1]} |E(x) - (\beta x + \alpha)| \leq |E''(z)|$ for $\beta = E'(0), \alpha = E(0)$, and some $z \in [-1, 1]$.
    Thus we can take $\epsilon = \max_{z \in [-1, 1]} |E''(z)|$.
    
    We solve directly for $\beta$:
    \begin{align*}
        \beta &= E'(0) = \frac{1}{\sigma} \biggl( \phi \left( \frac{-\mu}{\sigma} \right) + \frac{1}{\sqrt{m}} Q_1'\left( \frac{-\mu}{\sigma} \right) + \frac{1}{m} Q_2'\left( \frac{-\mu}{\sigma} \right) \biggr).
    \end{align*}

    We also solve for $|E''(z)|$, from which the result follows immediately:
    \begin{align*}
        |E''(z)| &= \frac{1}{\sigma^2} \biggl| \left( \frac{z - \mu}{\sigma} \right) \phi \left( \frac{z - \mu}{\sigma} \right) + \frac{1}{\sqrt{m}} Q_1'' \left( \frac{z - \mu}{\sigma} \right) + \frac{1}{m} Q_2'' \left( \frac{z - \mu}{\sigma} \right) \biggr| \\
        &= \frac{1}{\sigma^3} \phi \left( \frac{z - \mu}{\sigma} \right) \biggl| (z - \mu) - \frac{C_3}{6} H_4 \left( \frac{z - \mu}{\sigma} \right) - \frac{C_3^2}{72 \sigma} H_7 \left( \frac{z - \mu}{\sigma} \right) - \frac{C_4}{ 24 \sigma} H_5\left(\frac{z - \mu}{\sigma}\right)  \biggr|.
    \end{align*}
\end{proof}

\subsection{Leave-one-out Approximate Truthfulness}
We can now combine Lemmas $\ref{lem:approx-linear-implies-approx-truthful}$, \ref{lem:edgeworth}, and \ref{lem:approx-linear} with Condition \ref{cond:final-thm-2-forecasters} to give an approximate truthful parameter for event $t$.

We first state the following intermediate lemma.
\begin{lemma} \label{lem:bounded-ratio}
    If Condition \ref{cond:final-thm-2-forecasters} holds, $\frac{|\mu_i(\vect{\hat{r}}_i)|}{\sigma_i(\vect{\hat{r}}_i)} \leq 1$ for any best response $\vect{\hat{r}}_i$.
\end{lemma}
\begin{proof}
    Recall that the expected utility of forecaster $i$ as a function of her report is $\overline{U}(\vect{r}_i) = P_{\D_i}[\sum_t S(r_{it}, R_{jt}, Y_{t}) \leq 0]$; it follows by the Berry-Esseen theorem \citep{esseen1942,shevtsova2010improvement} that for any report $\vect{r}_i$,
    \begin{equation} \label{eq:expected-utility}
        \Bigl|\overline{U}(\vect{r}_i) - \Phi \left(\frac{-\mu_i(\vect{r}_i)}{\sigma_i(\vect{r}_i)} \right) \Bigr| \leq \frac{P_i(\vect{r}_i)}{\sigma_i(\vect{r}_i)^3}.
    \end{equation}
    By our conditions, for any report vector $\vect{r}_i$, $\overline{U}(\vect{r}_i) \leq \frac{1}{2} + \delta$. 
    Moreover, for forecaster $i$'s belief vector $\vect{p}_i$, $\overline{U}(\vect{p}_i) \geq  \frac{1}{2} - \delta$.
    Since any strategic report $\vect{\hat{r}}_i \neq \vect{p}_i$ must increase the expected utility from $\vect{p}_i$, it follows that $\frac{1}{2} - \delta \leq \overline{U}(\vect{\hat{r}}_i) \leq \frac{1}{2} + \delta$.
    Thus by Inequality \ref{eq:expected-utility}, $\Phi\left(-\frac{\mu_i(\vect{\hat{r}}_i)}{\sigma_i(\vect{\hat{r}}_i)}\right) \leq \frac{1}{2} + \delta + \frac{P_i(\vect{\hat{r}}_i)}{\sigma_i(\vect{\hat{r}}_i)^3}$ and $\Phi\left(-\frac{\mu_i(\vect{\hat{r}}_i)}{\sigma_i(\vect{\hat{r}}_i)}\right) \geq  \frac{1}{2} - \delta - \frac{P_i(\vect{\hat{r}}_i)}{\sigma_i(\vect{\hat{r}}_i)^3}$.
    Let $\hat{\delta} = \delta + \frac{P_i(\vect{\hat{r}}_i)}{\sigma_i(\vect{\hat{r}}_i)^3}$. 
    Then,
    \begin{align*}
        -\frac{\mu_i(\vect{\hat{r}}_i)}{\sigma_i(\vect{\hat{r}}_i)} &\geq \Phi^{-1} (1/2 - \hat{\delta}) \geq \sqrt{\frac{\pi}{8}} \text{logit}(1/2 - \hat{\delta}), \\
        -\frac{\mu_i(\vect{\hat{r}}_i)}{\sigma_i(\vect{\hat{r}}_i)} &\leq \Phi^{-1} (1/2 + \hat{\delta}) \leq \sqrt{\frac{\pi}{8}} \text{logit}(1/2 +  \hat{\delta}).
    \end{align*}
    Thus
    \[ \frac{|\mu_i(\vect{\hat{r}}_i)|}{\sigma_i(\vect{\hat{r}}_i)} \leq \sqrt{\frac{\pi}{8}} \log \left( \frac{1/2 + \hat{\delta}}{1/2 - \hat{\delta}} \right) \leq 1, \]
    where the last inequality holds since $\hat{\delta} \leq \delta + \frac{P_i}{\sigma_i^3} \leq 0.33.$
\end{proof}

\begin{lemma} \label{lem:take-one-out-approx-truthful}
    Assume Conditions \ref{cond:edgeworth} and \ref{cond:final-thm-2-forecasters} hold. 
    Then agent $i$ is $\gamma$-$\linf$ truthful for event $t$ with 
    \[\gamma^2/2 \geq e^{\mu_{it}^2/\sigma_{it}^2} \frac{\sigma_{it}^{-2} \left( |\mu_{it}| + A_{it} \right) + \sigma_{it} D_{it} m^{-1}} {1 - B_{it} \sigma_{it}^{-1} }, \]

    where 
    \begin{align*}
        A_{it} &= 1 + 5|C_{3,it}|  + \frac{6C_{3,it}^2}{\sigma_{it}} + \frac{|C_{4,it}|}{\sigma_{it}}, \\
        B_{it} &= 3|C_{3,it}| + \frac{2C_{3,it}^2}{\sigma_{it}} + \frac{2|C_{4,it}|}{\sigma_{it}}.
    \end{align*}
\end{lemma}

\begin{proof}
    By Lemma \ref{lem:edgeworth}, for large enough $m$ $\|G_{it} - E_{it} \|_{\infty} \leq D_{it}$.
    By Lemma \ref{lem:approx-linear}, $\|E_{it} - L \|_{\infty} \leq \epsilon_1$ for $L(x) = \beta x + \alpha$ (for $\beta$, $\alpha$, and $\epsilon_1$ defined by Equations \ref{eq:beta}, \ref{eq:alpha}, and \ref{eq:epsilon} under $\mu_{it}$ and $\sigma_{it}$).
    By the triangle inequality, it follows that $G_{it}$ is $(\beta, \alpha, \epsilon)$-approximately affine with $\epsilon = \epsilon_1 + \epsilon_2,$ where $\epsilon_2 = D_{it} m^{-1}$ for some positive constant $D_{it}$.
    By Lemma \ref{lem:approx-linear-implies-approx-truthful}, then, forecaster $i$ is $\gamma$-$\linf$ truthful for $\gamma = \sqrt{2\epsilon/\beta}$. 

    Note that for $\sigma_i \geq 4$,
    \begin{itemize}
        \item $\sigma_{it} \geq \sigma_i(\hat{r}_i) - 1 \geq \sigma_i(\hat{r}_i)/{\sqrt{2}}$,
        \item $\frac{|\mu_{it}|}{\sigma_{it}} \leq \frac{|\mu_i(\hat{r}_i)| + 1}{\sigma_i(\hat{r}_i) - 1} \leq 2$ (by Lemma \ref{lem:bounded-ratio}).
    \end{itemize}
    It follows that for any $z \in [-1, 1]$, $\left| \frac{z - \mu_{it}}{\sigma_{it}} \right| \leq 3$.
    Thus 
    \begin{align*}
        \beta &\geq \frac{1}{\sigma_{it}} \phi\left(\frac{-\mu_{it}}{\sigma_{it}}\right) \min_{y \in [-3, 3]} \biggl( 1 + \frac{C_{3,it}}{6 \sigma_{it}} H_3 \left(y \right) + \frac{C_{3,it}^2}{72 \sigma_{it}^2}  H_6 \left(y \right) + \frac{C_{4,it}}{24 \sigma_{it}^2} H_4 \left(y\right) \biggr) \\
        &= \frac{1}{\sigma_{it}} \phi\left(\frac{-\mu_{it}}{\sigma_{it}}\right) \left( 1 - \frac{B_{it}}{\sigma_{it}} \right)
    \end{align*}
    and
    \begin{align*}
        \epsilon_1 &\leq \frac{1}{\sigma_{it}^3} \phi \left(\frac{1 + |\mu_{it}|}{\sigma_{it}}\right) \max_{y \in [-3, 3]} \biggl( |\mu_{it}| + 1 + \biggl| \frac{C_{3,it}}{6} H_4 \left(y\right) - \frac{C_{3,it}^2}{72 \sigma_{it}} H_7 \left(y\right) - \frac{C_{4,it}}{ 24 \sigma_{it}} H_5 \left(y\right) \biggr| \biggr) \\
        &\leq \frac{1}{\sigma_{it}^3} (|\mu_{it}| + A_{it})\phi \left(\frac{1 + |\mu_{it}|}{\sigma_{it}}\right).
    \end{align*}

    The statement follows by calculating $\frac{\epsilon_1 + \epsilon_2}{\beta}$ and noting that $\phi \left(\frac{1 + |\mu_{it}|}{\sigma_{it}}\right) \Bigl/ \phi \left(\frac{-\mu_{it}}{\sigma_{it}}\right) \leq e^{\mu_{it}^2/\sigma_{it}^2}$ (in an abuse of notation, we absorb a constant of $2\pi$ into $D_{it}$).
\end{proof}

\subsection{Proof of Theorem \ref{thm:2-forecaster-truthful}}

Now we have the tools to prove our main theorem showing approximate truthfulness.
Specifically, we derive a lower bound on $\gamma$ depending on $\sigma_i$ that guarantees leave-one-out approximate truthfulness for all events $t$. 

\begin{proof}{(Theorem \ref{thm:2-forecaster-truthful})}
    By Lemma \ref{lem:take-one-out-approx-truthful}, to guarantee forecaster $i$ is $\gamma$-$\linf$ approximately truthful, we require for all $t$ that
    \[ \gamma \geq \sqrt{2} e^{\frac{\mu_{it}^2}{2\sigma_{it}^2}} \frac{\sigma_{it}^{-1} \left( |\mu_{it}| + A_{it} \right)^{1/2} + (\sigma_{it} D_{it} m^{-1})^{1/2}} {\left(1 - B_{it} \sigma_{it}^{-1} \right)^{1/2}}. \]
    
    Note that for $\sigma_i \geq 4$, $\sigma_{it} > \sigma_i(\hat{r}_i) - 1 \geq \frac{\sigma_i(\hat{r}_i)}{\sqrt{2}}.$ 
    Thus
    \begin{align*}
        \exp \left( \frac{|\mu_{it}|^2} {2 \sigma_{it}^2} \right) &\leq \exp \left( \frac{(|\mu_i(\hat{r}_i)| + 1)^2}{\sigma_i(\hat{r}_i)^2} \right) \\
        &\leq \exp \left( \frac{1}{\sigma_i(\hat{r}_i)^2} + \frac{2}{\sigma_i(\hat{r}_i)} + 1\right) \\
        &\leq 5.
    \end{align*}

    Meanwhile, note that $\kappa_{3, it} \leq \kappa_{3, i} + 8$, and $\kappa_{4, it} \leq \kappa_{4, i} + 32$.
    Thus $|C_{3, it}| \leq 2 |C_{3,i}| + \frac{16}{\sigma_i^2}$ and $C_{4, it} \leq 2C_{4,i} + \frac{64}{\sigma_i^2}$, so that
    \begin{align*}
        \sigma_{it}^{-1} \left( |\mu_{it}| + A_{it} \right)^{1/2} &= \sigma_{it}^{-1} \left( |\mu_{it}| + 1 + 5|C_{3,it}|  + \frac{6C_{3,it}^2}{\sigma_{it}} + \frac{|C_{4,it}|}{\sigma_{it}} \right)^{1/2} \\
        &\leq \frac{2}{\sigma_i(\hat{r}_i)} \left( |\mu_i(\hat{r}_i)| + 11 + 19|C_{3,i}| + 9C_{3,i}^2 + C_{4,i} \right)^{1/2} \\
        &\leq 2C\sigma_i^{-1/2},
    \end{align*}

    where $C = 4 + 5\sqrt{|C_{3,i}|} + 3|C_{3,i}| + \sqrt{C_{4,i}}.$

    Finally, note that 
    \begin{align*}
        \sigma_{it}^{-1/2} \left(3|C_{3,it}| + \frac{2C_{3,it}^2}{\sigma_{it}} + \frac{2|C_{4,it}|}{\sigma_{it}} \right) 
        &\leq 2\sigma_i(\hat{r}_i)^{-1/2} \left( 4 + 9|C_{3,i}| + 3C_{3,i}^2 + C_{4,i}  \right)^{1/2} \\
        &\leq 2\sigma_i^{-1/2} \left( 2 + 3\sqrt{|C_{3,i}|} + 2|C_{3,i}| + \sqrt{|C_{4,i}|} \right) \\
        &\leq 2C\sigma_i^{-1/2}.
    \end{align*}

    Now, we take $D^2 = \max_t D_{it}$; it follows that $D_{it} m^{-1} \leq D m^{-1}$; moreover, $\sigma_i m^{-1} \leq \sigma_i^{-1}$ (since $\sigma_i^2 \leq m$), so that

    \begin{align*}
        \sigma_{it} D_{it} m^{-1} &\leq (\sigma_i + 1) D^2 m^{-1} \\
        &\leq 2D^2 \sigma_i^{-1}.
    \end{align*}

    In total, then, we can guarantee $\gamma$-$\linf$ approximate truthfulness when
    \begin{align*}
        \gamma &\geq \frac{8(2C + 3D)\sigma_i^{-1/2}}{1 - 2C\sigma_i^{-1/2}} \\
        &= \frac{8(2C + 3D)}{\sigma_i^{1/2} - 2C}.
    \end{align*}
    The result follows.
\end{proof}

\section{Utility Functions with Ties} \label{appendix:ties}
We briefly discuss the assumption made in Section \ref{sec:approx-truthfulness} about continuity of report distributions.
Assuming score differences have a smooth distribution may seem counterintuitive, since forecasters are likely to report on some discrete scale. 
But in that case, as $m$ grows we still expect the probability of a tie to become small, and our analysis extends.

Formally, we note that the expected utility when the probability of ties is non-zero can be written as 
\begin{align*}
        \overline{U}(\vect{r}_i) = \Pr_{\D_i} &\left[ \sum_t S(r_{it}, R_{jt}, Y_t) \leq 0 \right] + \frac{1}{2} \Pr_{\D_i} \left[ \sum_t S(r_{it}, R_{jt}, Y_t) = 0 \right].
\end{align*}
If the score differences are drawn from lattice distributions with small enough granularity, we expect that the second probability should approach 0 as $m$ grows. 
In particular, if $\Pr_{\D_i} \left[ \sum_t S(r_{it}, R_{jt}, Y_t) = 0 \right] = O(1/m)$, we can conclude that $i$ is $(\beta, \alpha, \epsilon_1 + \epsilon_2 + \epsilon_3)$-approximately affine for $\epsilon_3 = O(1/m)$ (as a reminder, $\epsilon_1$ is incurred by the Taylor approximation error of the Edgeworth expansion, and $\epsilon_2$ is the Edgeworth approximation error). 
We can then apply Lemma \ref{lem:approx-linear-implies-approx-truthful} with our new error term and achieve the same approximate truthfulness result.

\end{document}